%% file: acl_latex.tex
\documentclass[11pt]{article}

\usepackage[final]{acl}

\usepackage{times}
\usepackage{latexsym}

\usepackage[T1]{fontenc}

\usepackage[utf8]{inputenc}

\usepackage{microtype}

\usepackage{inconsolata}

\usepackage{graphicx}

\usepackage{enumitem}
\usepackage{amsmath}
\usepackage{amssymb}
\usepackage{amsthm}

\theoremstyle{plain}
\newtheorem{theorem}{Theorem}
\newtheorem{lemma}{Lemma}
\newtheorem{proposition}{Proposition}

\theoremstyle{definition}

\theoremstyle{remark}
\newtheorem{remark}{Remark}

\usepackage[table]{xcolor}
\usepackage[utf8]{inputenc}
\usepackage{textcomp}
\usepackage{algorithm}
\usepackage{algorithmic}
\usepackage{float}
\usepackage{booktabs}
\usepackage[most]{tcolorbox}

\newtcolorbox{promptbox}[1]{
  enhanced,
  colback=gray!5,
  colframe=blue!60,
  boxrule=0.8pt,
  arc=4pt,
  left=6pt,
  right=6pt,
  top=6pt,
  bottom=6pt,
  fonttitle=\bfseries,
  title={#1},
  breakable
}

\title{SFAD: Speculative Factuality-Aware Decoding}

\author{
  \textbf{Guanqiao Chen\textsuperscript{1,2}},
  \textbf{Di Wang\textsuperscript{3,4}},
  \textbf{Lijie Hu\textsuperscript{1,*}}
\\
  \textsuperscript{1}MBZUAI
\\
  \textsuperscript{2}University of Science and Technology of China
\\
  \textsuperscript{3}Provable Responsible AI and Data Analytics (PRADA) Lab
\\
  \textsuperscript{4}King Abdullah University of Science and Technology
\\
  \textsuperscript{*}Corresponding author.
}

\begin{document}
\maketitle
\begin{abstract}
  As one of the most critical challenges in large language models, contextual faithfulness directly determines their reliability in knowledge-intensive applications. This task is particularly challenging as it requires balancing factual consistency with generation efficiency. Contrastive decoding methods require dual forward passes (with and without context) to compare model outputs, doubling inference computational overhead, while post-training alignment demands extensive reinforcement learning with substantial computational overhead. To address this challenge, we present \textbf{SFAD}, a speculative decoding framework that enhances contextual faithfulness without inference degradation. We first construct \textbf{ConFide}, a preference dataset with fine-grained atomic perturbations, to train a context-faithful draft model via Direct Preference Optimization. During inference, Epistemic Friction detects potential hallucinations by quantifying distributional tension weighted by specialist certainty. When friction exceeds the threshold, Asymmetric Logit Steering refines the target distribution through residual-based logit injection; otherwise, standard speculation proceeds. Extensive experiments demonstrate that SFAD substantially improves faithfulness while achieving $2.48\times$ speedup, offering a practical solution for efficient LLMs.
\end{abstract}

\input{Section/1_Intro}

\input{Section/2_Related}
\input{Section/3_DataConstruction}

\input{Section/4_SFAD}
\input{Section/5_Experiment}

\input{Section/6_Analysis}
\input{Section/7_Conclusion}

\input{Section/8_impact_statement}
\input{Section/Ack}

\bibliography{custom}

\appendix
\input{Section/appendix_two}

\end{document}

%% file: Section/1_Intro.tex
\section{Introduction}
\begin{figure}[t]
\centering
\includegraphics[width=\columnwidth]{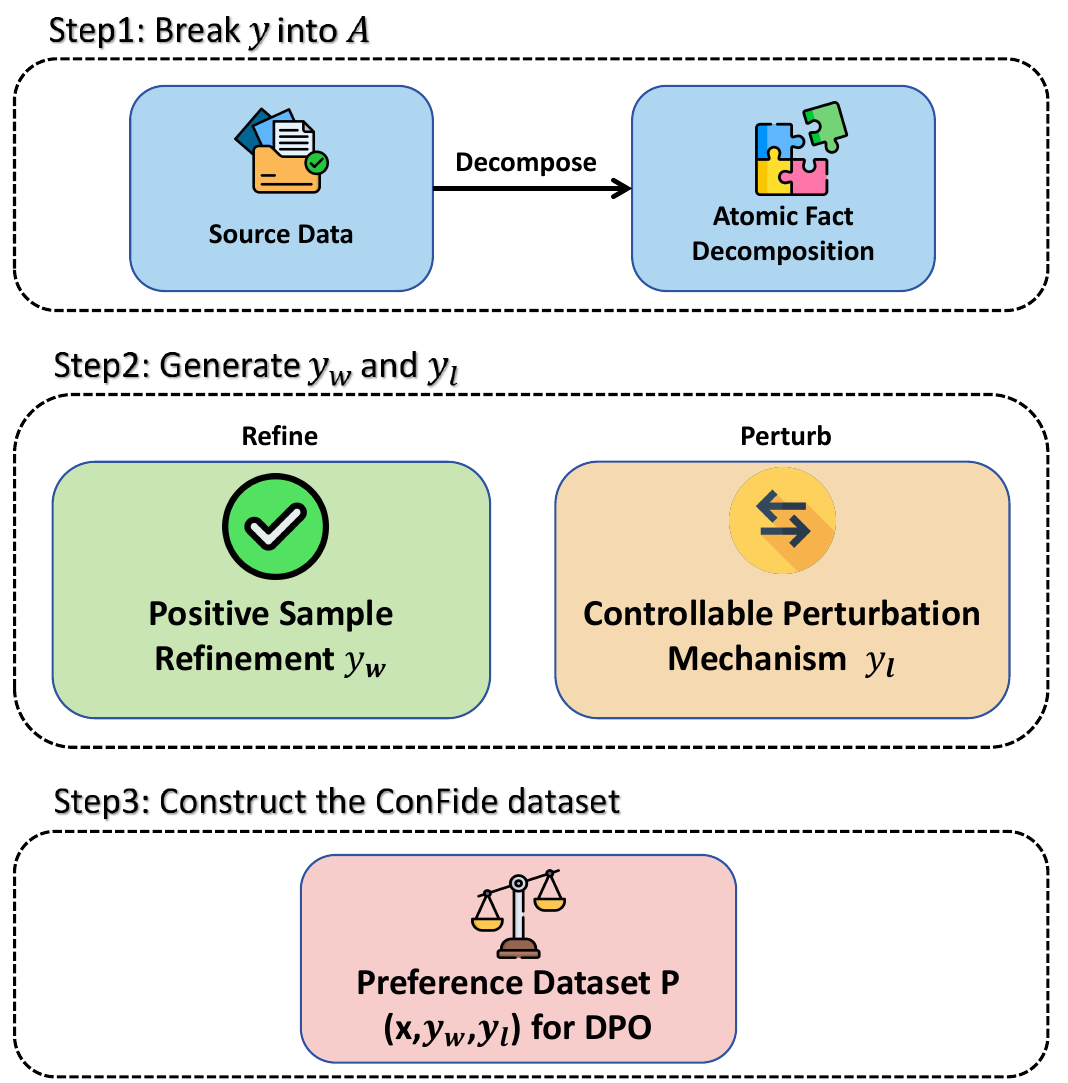}
\caption{The ConFide pipeline for factuality-aware preference data construction.}
\label{fig:confide-framework}
\end{figure}

Large Language Models (LLMs)~\cite{guo2025deepseek,achiam2023gpt} have achieved remarkable success by leveraging vast parametric knowledge acquired during pretraining~\cite{petroni-etal-2019-language, roberts-etal-2020-much}. However, this internal knowledge is inherently static and constrained by the training distribution, rendering it prone to becoming outdated or incomplete~\cite{cheng2024multihop,Li_2026_CVPR,pmlr-v267-zhang25aq,cheng2024leveraging}. To address this limitation, retrieval-augmented generation (RAG)~\cite{qin2025tool, pmlr-v119-guu20a,cheng-etal-2025-compke} and external tool integration have become increasingly prevalent, intensifying the demand for models to prioritize provided context over parametric priors. Yet LLMs frequently exhibit hallucinations~\cite{huang-etal-2025-improving, ji2023survey,cheng-etal-2025-codemenv} due to \textit{knowledge conflicts}, where models favor internal knowledge over task-specific external contexts~\cite{xie2024adaptive, chen-etal-2022-rich,zhang2025modalities}, undermining their contextual faithfulness in knowledge-intensive applications.

\begin{figure*}[t]
    \centering
    \includegraphics[width=\textwidth]{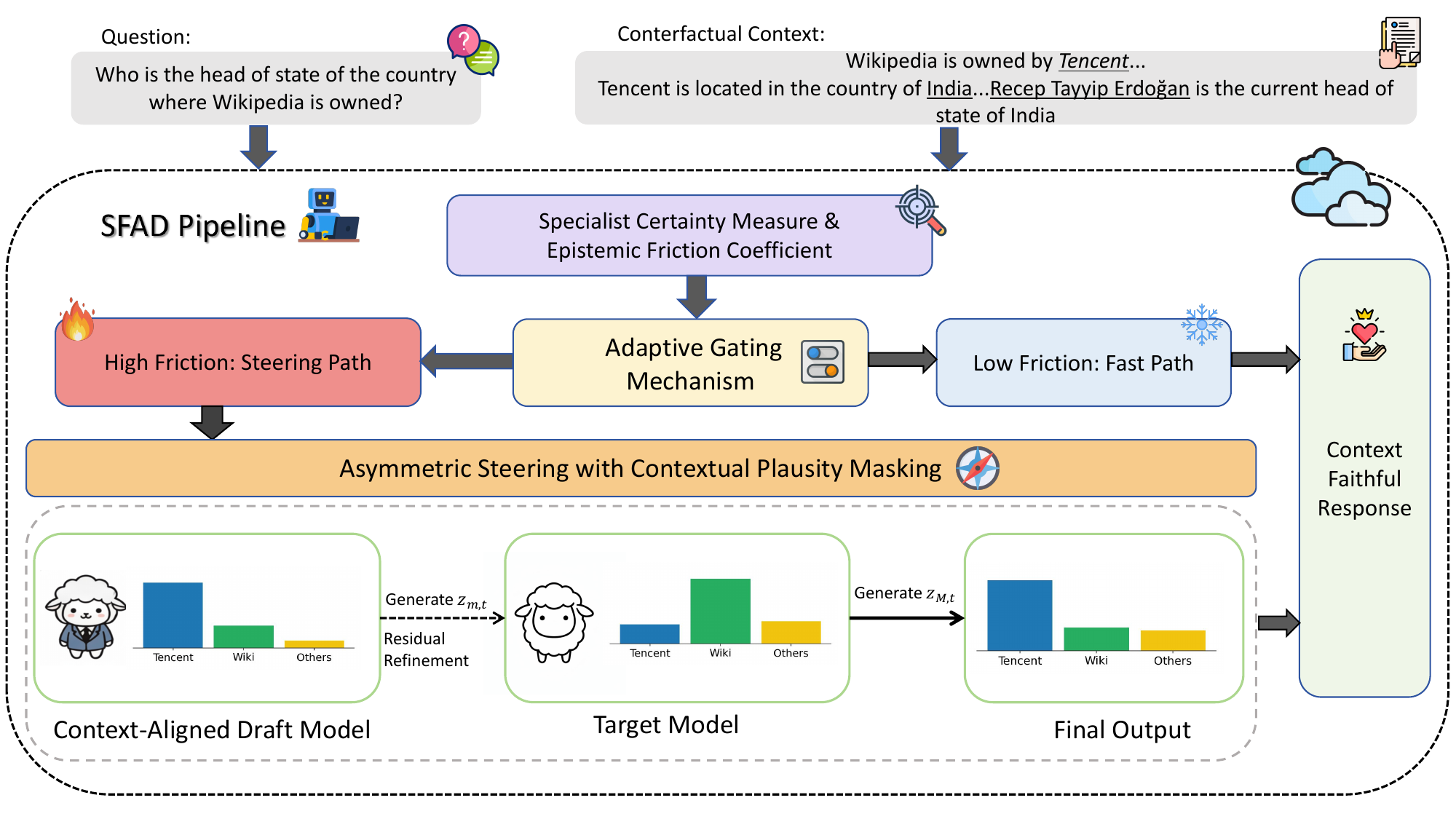}
    \caption{SFAD framework.}
    \label{fig:pipeline}
\end{figure*}

Enhancing contextual faithfulness while maintaining inference efficiency presents significant challenges. Decoding-based methods~\cite{shi-etal-2024-trusting, xu2023context, khandelwal-etal-2025-cocoa,yang2025tracing} contrast logits from contextualized inputs against uncontextualized baselines to amplify evidence-aligned signals. However, this requires dual forward, effectively doubling computational cost and halving generation speed. Post-training methods~\cite{yan-etal-2025-rpo, wang2025infinite,yang2025investigating,zhou-etal-2026-flattery} employ reinforcement learning to align models with contextual evidence, but typically require extensive compute and large-scale preference data. These limitations motivate the need for approaches that enhance faithfulness without compromising efficiency or requiring substantial resources.

To resolve these problems, we propose \textbf{SFAD}, a speculative decoding framework that unifies contextual faithfulness enhancement with speculative acceleration. Our key insight is to leverage a \textit{context-faithful draft model} as a factuality sentinel: by training a smaller model to prioritize contextual evidence, we enable it to detect and correct hallucinations in the larger target model during speculative decoding. Specifically, we first construct \textbf{ConFide}, a preference dataset that improves contextual faithfulness by generating diverse hard-negative samples through atomic decomposition and controllable perturbations. Training the draft model on ConFide and ConFiQA~\cite{bi-etal-2025-context} via Direct Preference Optimization instills strong contextual loyalty. During inference, SFAD introduces \textit{Epistemic Friction}, a metric that quantifies distributional tension weighted by specialist certainty to dynamically detect knowledge conflicts. When friction signals potential hallucinations, \textit{Asymmetric Logit Steering} selectively refines the target distribution through residual-based injection; otherwise, standard speculative decoding proceeds unmodified. This adaptive mechanism maintains the $2\sim 3\times$ speedup of speculative decoding while substantially improving contextual faithfulness.

Our contributions are threefold: 
\textbf{(1) Dataset: } we introduce ConFide, a fine-grained preference dataset designed for training context-faithful draft models through atomic-level perturbations; 
\textbf{(2) Framework: } we propose \textbf{SFAD}, the first speculative decoding framework specifically designed for hallucination mitigation, which features \textit{Epistemic Friction} for conflict detection and \textit{Asymmetric Logit Steering} for selective correction; and 
\textbf{(3) Evaluation: } extensive experiments demonstrate that \textsc{SFAD} achieves substantial faithfulness improvements across diverse benchmarks while delivering a $2.48\times$ speedup, approaching the performance of models $5\times$ larger with minimal computational overhead.

%% file: Section/2_Related.tex
\section{Related Work}

\subsection{Hallucinations in LLMs}
Hallucinations\cite{huang2025survey} occur when LLM outputs appear plausible yet deviate from factual or contextual knowledge~\cite{kaddour2023challenges, ji2023survey}. These are typically bifurcated into \textit{factuality hallucinations}, which contradict real-world facts based on internal parametric knowledge~\cite{min-etal-2023-factscore, wei2024measuring}, and \textit{faithfulness hallucinations}, characterized by inconsistency with the provided input or grounding documents~\cite{wang2026truth,wan-etal-2023-faithfulness,hu2025monica,guo2025benchmarking}. Mitigation strategies span the model lifecycle: training-phase efforts focus on data curation and grounding~\cite{pan2024unifying}, while inference-stage methods employ confidence estimation~\cite{huang2025look}, knowledge retrieval~\cite{feng2024retrieval}, and editing~\cite{ali-etal-2025-mqa,wang2025lifecycle}. Despite these advancements, LLMs remain prone to inconsistent outputs in context-rich tasks such as RAG and summarization~\cite{xu-etal-2024-knowledge-conflicts, li-yu-2025-summary}. \textbf{Our work leverages the Speculative Decoding framework to ensure context-consistent generation.}

\subsection{Speculative Decoding}
Speculative decoding accelerates LLM inference by verifying draft tokens generated by a smaller model against the target LLM~\cite{pmlr-v202-leviathan23a, chen2023accelerating}. Drafting strategies have evolved from independent models~\cite{xia-etal-2023-speculative} to architectural extensions such as recurrent feature utilization~\cite{pmlr-v235-li24bt} and diffusion-based generation~\cite{li-etal-2026-diffuspec}. Efficiency has been further optimized through tree-based parallel verification~\cite{miao2024specinfer}, sparse Mixture-of-Experts (MoE) acceleration~\cite{huang2025moesd}, and specialized reinforcement learning system enhancements~\cite{chen2026respec}. Recent research has also tailored SD for diverse and complex scenarios, including long-context retrieval-augmentation~\cite{sun2024triforce, sadhukhan2025magicdec, yang-etal-2026-longspec, pmlr-v267-chen25s}, multi-sample inference tasks~\cite{li-etal-2025-speculative}, and memory-efficient inference via quantized caching~\cite{pmlr-v267-tiwari25b}. Beyond performance, safety-awareness has been integrated into the speculative framework to ensure alignment~\cite{wang-etal-2025-speculative}. Unlike these performance-driven or safety-centric works, \textbf{SFAD is the first speculative framework to target hallucination mitigation.}

%% file: Section/3_DataConstruction.tex
\section{Aligning Draft Models for Context-Faithfulness via ConFide}
\label{sec:confide_alignment}

To enhance the draft model's generalization across diverse hallucination types, we construct \textbf{ConFide}, a fine-grained preference dataset that systematically diversifies error patterns in negative samples and stylistic variations in positive samples. As illustrated in Figure~\ref{fig:confide-framework}, our pipeline leverages atomic fact decomposition and controllable perturbations to produce high-quality contrastive pairs.

\subsection{Data Construction Pipeline}
\noindent\textbf{Source Data and Problem Formulation.}
We build upon the LLM-AggreFact~\cite{tang-etal-2024-minicheck} and CG2C~\cite{lei-etal-2025-factcg} datasets, which provide fact-intensive samples covering diverse hallucination patterns suitable for atomic decomposition. Our initial dataset is defined as $\mathcal{D} = \{(x^{(i)}, y_{\text{ref}}^{(i)}, l^{(i)})\}_{i=1}^N$, where $x$ is the source context, $y_{\text{ref}}$ is the candidate response, and $l \in \{0, 1\}$ is the binary faithfulness label. Our goal is to construct a preference dataset $\mathcal{P} = \{(x, y_w, y_l)\}$ for DPO training, where $y_w$ is faithful to context $x$ and $y_l$ contains specific hallucinations.

\noindent\textbf{Atomic Fact Decomposition.}
We decompose each response into verifiable atomic facts via a decomposition function $f_{\text{dec}}$:
\begin{equation}
    \mathcal{A} = f_{\text{dec}}(y) = \{a_1, a_2, \dots, a_m\},
\end{equation}
where each atomic fact $a_j$ represents a minimal verifiable claim.

\noindent\textbf{Negative Sample Generation.}
We apply a Controllable Perturbation Mechanism to transform atomic facts into corrupted versions. Given $a_j = (s_j, r_j, o_j)$ as a subject-relation-object triple, we define three perturbation operators:
\begin{equation}
\phi(a_j) = \begin{cases}
(s_j, r_j, o'_j), & \tau_{\text{ent}} \\
(s_j, r_j, o_j + \epsilon), & \tau_{\text{num}} \\
(s_j, \neg r_j, o_j), & \tau_{\text{rel}}
\end{cases}
\end{equation}
where $\tau_{\text{ent}}$, $\tau_{\text{num}}$, and $\tau_{\text{rel}}$ denote entity swap, numerical distortion, and relation inversion, respectively. Here, $o'_j$ is a confusing entity drawn from context $x$, $\epsilon$ is calibrated noise preserving numerical plausibility, and $\neg r_j$ denotes logical negation of relation $r_j$. The perturbed facts are reconstructed into a fluent response $y_l = f_{\text{rec}}(\mathcal{A}')$, yielding hallucinated yet linguistically natural negative samples.

\noindent\textbf{Positive Sample Refinement.}
Winning samples $y_w$ are constructed based on the faithfulness label $l$. If $l=1$, we apply paraphrasing $\pi_{\text{para}}$ to generate semantically equivalent but syntactically varied sentences. If $l=0$, a teacher model (GPT-4o~\cite{hurst2024gpt}) corrects $y_{\text{ref}}$ against context $x$ to produce a strictly faithful $y_w$. This ensures winning samples reflect genuine factual alignment rather than surface-level patterns.

\noindent\textbf{Preference Dataset Construction.}
Pairing the perturbed $y_l$ with the refined $y_w$ yields the final preference dataset $\mathcal{P}$, where negative samples exhibit fine-grained, fluent hallucinations and positive samples maintain strict contextual fidelity.

\subsection{Draft Model Optimization via DPO}
We optimize the draft model $\pi_\theta$ via the DPO objective~\citep{rafailov2023direct}, which directly maximizes the preference margin between faithful responses $y_w$ and hallucinated responses $y_l$ without requiring an explicit reward model. The gradient signal encourages $\pi_\theta$ to concentrate probability mass on tokens strictly verifiable against the context $x$, controlled by hyperparameter $\beta$ relative to the frozen reference model $\pi_{\text{ref}}$. 

%% file: Section/4_SFAD.tex
\section{SFAD: Speculative Factuality-Aware Decoding}
\label{sec:SFAD}
We formulate \textbf{Speculative Factuality-Aware Decoding (SFAD)} as a dynamic inference framework designed to decouple token verification from distribution correction. Figure~\ref{fig:pipeline} illustrates the overall pipeline. Let $M$ and $m$ denote the target (generalist) and draft (specialist) models with parameters $\theta_M$ and $\theta_m$, respectively. At each time step $t$, given the prefix context $x_{<t}$, the models produce logit vectors $\mathbf{z}_{M,t}, \mathbf{z}_{m,t} \in \mathbb{R}^{|\mathcal{V}|}$. Our framework is governed by a five-step formalism that adaptively steers the generation process when the specialist detects potential factual inconsistencies. Theoretical analysis and proofs are provided in Appendix~\ref{app:theory}.

\subsection{Specialist Certainty Measure}
To prevent noise injection from an uncertain draft model, we first quantify its internal confidence. Unlike standard entropy measures, we require a normalized metric that penalizes high-entropy (uncertain) distributions. We define the \textbf{Specialist Certainty} $\kappa_t$ as:
\begin{gather}
    \kappa_t = \left( 1 - \frac{\mathbb{H}(\mathbb{P}_m(\cdot|x_{<t}))}{\log |\mathcal{V}|} \right)^\gamma, \\
    \text{where } \mathbb{P}_m = \text{Softmax}(\mathbf{z}_{m,t}). \nonumber
\end{gather}
where $\mathbb{H}(\cdot)$ denotes the Shannon entropy, $|\mathcal{V}|$ is the vocabulary size, and $\gamma \geq 1$ is a sharpening coefficient. Here, $\kappa_t \to 1$ implies the specialist is highly confident in its prediction, serving as a necessary precondition for intervention.

\subsection{Epistemic Friction Coefficient}
To distinguish factual conflicts from benign linguistic diversity, we propose the \textbf{Epistemic Friction} $\mathcal{F}_t$, which captures the distributional tension between the generalist and the specialist. It is formulated as the Jensen-Shannon (JS) divergence weighted by the specialist's certainty:
\begin{equation}
    \mathcal{F}_t = \mathcal{D}_{\text{JS}} \left( \mathbb{P}_M(\cdot|x_{<t}) \, \Vert \, \mathbb{P}_m(\cdot|x_{<t}) \right) \cdot \kappa_t
\end{equation}
$\mathcal{F}_t$ acts as a detector for \textit{confident hallucinations}. High friction occurs if and only if the models disagree significantly ($\mathcal{D}_{\text{JS}}$ is high) \textit{and} the specialist is factually convinced ($\kappa_t$ is high).

\subsection{Adaptive Gating Mechanism}
To ensure computational efficiency and avoid over-correction, SFAD employs a soft gating scalar $\lambda_t \in [0, 1]$ to govern the intensity of the steering. It is derived via a shifted sigmoid activation:
\begin{equation}
    \lambda_t = \sigma\left(\beta \cdot (\mathcal{F}_t - \tau)\right) = \frac{1}{1 + \exp\left(-\beta(\mathcal{F}_t - \tau)\right)}
\end{equation}
where $\tau$ is the friction threshold and $\beta$ controls the transition sharpness. This creates a switching regime: when $\mathcal{F}_t \ll \tau$, $\lambda_t \to 0$ (Standard Speculation); when $\mathcal{F}_t \gg \tau$, $\lambda_t \to 1$ (Steering Mode).

\subsection{Asymmetric Steering with Contextual Plausibility Masking}
When intervention is triggered, we modify the target logits. To ensure that the specialist only intervenes when its predictions are linguistically plausible within the current sequence, we introduce a \textbf{Contextual Plausibility Mask (CPM)}. We define the plausible set $\mathcal{V}_{CPC}$ as:
\begin{equation}
    \mathcal{V}_{CPC} = \left\{ v \in \mathcal{V} : \mathbb{P}_m(v|x_{<t}) \geq \eta p_t^{\max} \right\}
\end{equation}
where $p_t^{\max}$ is the maximum probability assigned by $\mathbb{P}_m(\cdot|x_{<t})$, and $\eta \in (0, 1)$ is a plausibility threshold. 

Inspired by previous work~\cite{bachmann2025judge}, which demonstrated that alignment-based verification in speculative decoding rejects many high-quality tokens due to distribution mismatches, we extend this insight beyond token-level acceptance. Instead of solely adapting verification schemes, we intervene at the logit level to inject factuality-aware corrections, ensuring both efficiency and reduced hallucinations.The corrected logit vector $\mathbf{z}^*_t$ is computed as:
\begin{equation}
\label{eq:steering}
    \mathbf{z}^*_t = \mathbf{z}_{M,t} + \lambda_t \cdot \text{ReLU}(\mathbf{z}_{m,t} - \mathbf{z}_{M,t}) \cdot \mathbb{I}(x \in \mathcal{V}_{CPC})
\end{equation}
The $\text{ReLU}(\cdot)$ operator enforces \textit{unidirectional knowledge injection}, allowing the specialist to boost correct entities without penalizing the generalist's linguistic fluency. The $\mathbb{I}(x \in \mathcal{V}_{CPC})$ term acts as a safety guard, ensuring the specialist only steers the target toward tokens that maintain "compositional plausibility."

\subsection{Hybrid Decoding Policy}
Finally, the next token $x_t$ is determined by a hybrid policy that switches between the corrected distribution and standard speculative verification:
\begin{equation}
\label{eq:hybrid}
    x_t \sim 
    \begin{cases} 
    \text{Softmax}(\mathbf{z}^*_t) & \text{if } \mathcal{F}_t \geq \tau \quad \text{(Steering Path)} \\
    \text{Verify}(\tilde{x}_t, \mathbb{P}_M) & \text{if } \mathcal{F}_t < \tau \quad \text{(Fast Path)}
    \end{cases}
\end{equation}
In the \textbf{Steering Path}, we sample from the steered logits, effectively overriding the draft. In the \textbf{Fast Path}, we perform standard rejection sampling using the draft token $\tilde{x}_t$ and the original target distribution $\mathbb{P}_M$, maintaining the speed guarantees of speculative decoding.

%% file: Section/5_Experiment.tex
\section{Experiments}
\subsection{Experimental Setup}
\label{sec:setup}

\noindent\textbf{Models and Configurations.} We implement the SFAD framework using the Qwen3 family~\cite{yang2025qwen3}. Specifically, Qwen3-1.7B is utilized as the draft model, which is fine-tuned via Direct Preference Optimization  to serve as the expert model.The target model is Qwen3-14B.

\noindent\textbf{Data Construction for DPO.} 
The DPO training set for the draft model consists of approximately 36K curated samples, combining two complementary sources. First, we directly adopt 18K instances from \textbf{ConFiQA}~\cite{bi-etal-2025-context}, a benchmark focusing on multi-hop knowledge conflicts. Second, we construct our \textbf{ConFide} dataset using the atomic perturbation pipeline described in Section~\ref{sec:confide_alignment}. Specifically, ConFide is synthesized from two source datasets: 12K samples from LLM-AggreFact~\cite{tang-etal-2024-minicheck}, and 6K samples from the CG2C dataset~\cite{lei-etal-2025-factcg}.

\noindent\textbf{Baselines.} 
We compare SFAD against two categories of baselines to demonstrate its superiority in both contextual faithfulness and inference efficiency. The first category consists of decoding-level methods applied to our target model, Qwen3-14B: Greedy Decoding, Context-Aware Decoding (CAD)~\cite{xu2023context}, COIECD~\cite{yuan-etal-2024-discerning}, and AdaCAD~\cite{wang-etal-2025-adacad}. The second category includes a frontier model, Llama-3.1-70B-Instruct~\cite{llama3.1}, which has $5\times$ the parameters of Qwen3-14B, serving as a high-performance reference point.

\noindent\textbf{Evaluation Tasks and Datasets.} 
We evaluate across five task categories: 
(1) \textit{Factual Retrieval} on HotpotQA~\cite{yang-etal-2018-hotpotqa}, PopQA~\cite{mallen-etal-2023-trust}, and TriviaQA~\cite{joshi-etal-2017-triviaqa}; 
(2) \textit{Abstractive Faithfulness} via summarization on TofuEval~\cite{tang-etal-2024-tofueval} and XSum~\cite{narayan-etal-2018-dont}; 
(3) \textit{Extended Generation} on CLAPNQ~\cite{rosenthal-etal-2025-clapnq}, ExpertQA~\cite{malaviya-etal-2024-expertqa}, and HAGRID~\cite{kamalloo2023hagrid}; 
(4) \textit{Knowledge Conflicts} using 200 held-out instances from LLM-AggreFact; 
and (5) \textit{General Capabilities} on GSM8K~\cite{cobbe2021training} and Just-Eval~\cite{lin2024unlocking}.

\noindent\textbf{Quality and Efficiency Metrics.} 
For generative quality, we use standard metrics tailored to each task: Exact Match (EM) for retrieval tasks; AlignScore, BERT-P, and ROUGE-L for summarization; and FaithScore (computed via MiniCheck~\cite{tang-etal-2024-minicheck}) for long-form QA. For knowledge conflict analysis, we adopt ConFiQA~\cite{bi-etal-2025-context} metrics: Context-faithful Frequency ($P_c$), Original Factual Frequency ($P_o$), and Memory Reliance ($M_R$). To quantify the inference acceleration of SFAD, we define the Average Token Generation Acceleration (ATGA), measured in an end-to-end, context-aware setting:
\begin{equation}
\label{eq:atga}
    \text{ATGA} = \frac{\text{Avg. token generation time w/o SFAD}}{\text{Avg. token generation time w/ SFAD}}
\end{equation}

\subsection{Main Results}

We present comprehensive evaluations across three categories of context-intensive tasks: Foundation QA, Summarization, and Long-Form QA.

SFAD demonstrates consistent metric improvements across all evaluated benchmarks, optimizing both task accuracy and inference efficiency. While standard decoding-level interventions typically incur latency penalties due to redundant forward passes, SFAD integrates a speculative decoding framework that streamlines the generation process. This efficiency is driven by our design’s focus on enhancing contextual consistency, which allows the model to leverage evidentiary alignment to accelerate token production. In Foundation QA, the method mitigates the internal knowledge limitations of the base model, a capability that extends to long-form generation where SFAD addresses the inherent trade-off between lexical fluency and factual grounding. By prioritizing context-aligned tokens, SFAD substantially narrows the performance gap between the 14B base model and the $5\times$ larger Llama-3.1-70B frontier model. These results underscore the capacity of SFAD to maintain sustained evidentiary grounding and long-range coherence with reduced computational costs, avoiding the need for task-specific fine-tuning or increased parameter overhead.

\begin{table}[t]
\centering
\small
\caption{Performance comparison on Foundation QA. \textbf{Rel. Latency} denotes the inference time relative to the Qwen3-14B vanilla baseline.}
\label{tab:foundation_qa}
\resizebox{\columnwidth}{!}{%
\begin{tabular}{lcccc}
\toprule
\textbf{Method} & \textbf{TriviaQA} & \textbf{HotpotQA} & \textbf{PopQA} & \textbf{Rel. Latency} \\
\midrule
\multicolumn{5}{l}{\textit{Vanilla Baseline}} \\
Qwen3-14B & 53.87 & 41.77 & 78.21 & 1.00$\times$ \\
\midrule
\multicolumn{5}{l}{\textit{Decoding-level Baselines (Qwen3-14B)}} \\
CAD & 41.43 & 39.51 & 71.29 & 2.00$\times$ \\
AdaCAD & 82.11 & 45.63 & 77.39 & 2.15$\times$ \\
COIECD & \underline{83.07} & \underline{45.63} & 76.29 & 2.40$\times$ \\
\midrule
\multicolumn{5}{l}{\textit{Frontier Model}} \\
Llama-3.1-70B & \textbf{90.20} & \textbf{56.11} & \underline{86.11} & 4.85$\times$ \\
\midrule
\rowcolor[HTML]{D9EAD3}
\textbf{SFAD (Ours)} & \underline{85.12} & \underline{52.19} & \textbf{86.39} & \textbf{0.82$\times$} \\
\bottomrule
\end{tabular}%
}
\end{table}

\begin{table}[t]
\centering
\small
\caption{Summarization performance on XSum and TofuEval. SFAD achieves superior factuality while reducing inference overhead.}
\label{tab:summarization}
\resizebox{\columnwidth}{!}{%
\begin{tabular}{lccccc}
\toprule
& \multicolumn{3}{c}{\textbf{XSum}} & \textbf{TofuEval} & \\
\cmidrule(lr){2-4} \cmidrule(lr){5-5}
\textbf{Method} & \textbf{R-L} & \textbf{BERT-P} & \textbf{AlignScore} & \textbf{AlignScore} & \textbf{Rel. Latency} \\
\midrule
\multicolumn{6}{l}{\textit{Vanilla Baseline}} \\
Qwen3-14B & 13.67 & 91.67 & 72.86 & 59.84 & 1.00$\times$ \\
\midrule
\multicolumn{6}{l}{\textit{Decoding-level Baselines (Qwen3-14B)}} \\
CAD & 14.59 & 93.65 & 84.34 & 83.23 & 2.00$\times$ \\
AdaCAD & 14.91 & \textbf{94.29} & \underline{85.81} & \underline{85.07} & 2.20$\times$ \\
COIECD & 13.65 & 91.04 & 73.81 & 60.86 & 2.45$\times$ \\
\midrule
\multicolumn{6}{l}{\textit{Frontier Model}} \\
Llama-3.1-70B & \underline{16.35} & \underline{94.12} & \textbf{87.48} & 87.31 & 5.10$\times$ \\
\midrule
\rowcolor[HTML]{D9EAD3}
\textbf{SFAD (Ours)} & \textbf{16.32} & 93.97 & 87.37 & \textbf{87.53} & \textbf{0.85$\times$} \\
\bottomrule
\end{tabular}%
}
\end{table}

\begin{table}[t]
\centering
\small
\caption{Long-Form QA results. SFAD significantly enhances faithfulness with higher efficiency than the vanilla model.}
\label{tab:longform_qa}
\resizebox{\columnwidth}{!}{%
\begin{tabular}{lcccccc|c}
\toprule
& \multicolumn{2}{c}{\textbf{CLAPNQ}} & \multicolumn{2}{c}{\textbf{ExpertQA}} & \multicolumn{2}{c}{\textbf{HAGRID}} & \\
\cmidrule(lr){2-3} \cmidrule(lr){4-5} \cmidrule(lr){6-7}
\textbf{Method} & \textbf{R-L} & \textbf{Faith} & \textbf{R-L} & \textbf{Faith} & \textbf{R-L} & \textbf{Faith} & \textbf{Rel. Latency} \\
\midrule
\multicolumn{8}{l}{\textit{Vanilla Baseline}} \\
Qwen3-14B & 17.12 & 59.73 & 31.56 & 51.29 & 16.96 & 57.63 & 1.00$\times$ \\
\midrule
\multicolumn{8}{l}{\textit{Decoding-level Baselines (Qwen3-14B)}} \\
CAD & 18.23 & 60.24 & 33.58 & 53.47 & 17.89 & 58.20 & 2.00$\times$ \\
AdaCAD & 18.43 & 62.37 & 32.87 & 54.76 & 17.12 & 59.90 & 2.18$\times$ \\
COIECD & \underline{18.56} & 61.96 & \underline{33.14} & \underline{56.32} & \underline{17.34} & \underline{59.76} & 2.42$\times$ \\
\midrule
\multicolumn{8}{l}{\textit{Frontier Model}} \\
Llama-3.1-70B & \textbf{42.15} & \textbf{92.45} & \textbf{46.10} & \textbf{72.40} & \textbf{52.07} & \textbf{82.20} & 5.30$\times$ \\
\midrule
\rowcolor[HTML]{D9EAD3}
\textbf{SFAD (Ours)} & \underline{41.34} & \underline{90.93} & \underline{43.79} & \underline{71.13} & \underline{51.97} & \underline{81.99} & \textbf{0.78$\times$} \\
\bottomrule
\end{tabular}%
}
\end{table}

%% file: Section/6_Analysis.tex
\section{Analysis}
\subsection{Effectiveness of ConFide and DPO }
\label{Data_DPO}

To isolate the contributions of our data construction pipeline and alignment strategy, we evaluate three draft model variants on 200 held-out knowledge-conflict instances from LLM-AggreFact. As illustrated in Figure~\ref{fig:dpo_effectiveness}, the transition from supervised fine-tuning (SFT) to preference optimization (DPO) triggers a fundamental shift in model behavior: while ConFide+SFT remains heavily anchored to internal priors (high $M_R$), DPO-trained variants demonstrate a superior ability to suppress parametric bias in favor of contextual evidence. 

Notably, ConFide+DPO consistently outperforms ConFiQA+DPO across all metrics, validating the efficacy of our atomic perturbation mechanism. By exposing the model to hard negatives—such as entity swaps and relation inversions—ConFide provides a more granular discriminative signal that sharpens the model's sensitivity to factual nuances. This contrastive optimization effectively penalizes hallucinated completions that appear superficially plausible but lack contextual grounding. \textbf{These results confirm that our data construction pipeline combined with preference optimization is crucial for training context-faithful draft models capable of effective speculative decoding.}

\begin{figure}[t]
\centering
\includegraphics[width=\columnwidth]{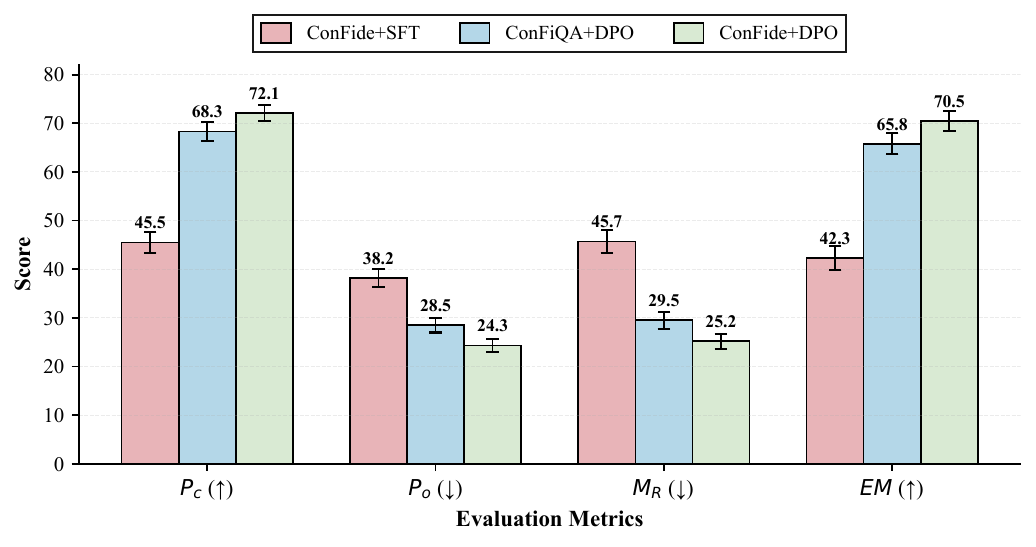}
\caption{Impact of training strategies on knowledge conflict resolution. 
ConFide+DPO consistently outperforms baselines across context faithfulness 
($P_c$, $EM$) and memory reliance ($M_R$, $P_o$) metrics.}
\label{fig:dpo_effectiveness}
\end{figure}

\begin{figure}[t]
\centering
\includegraphics[width=\columnwidth]{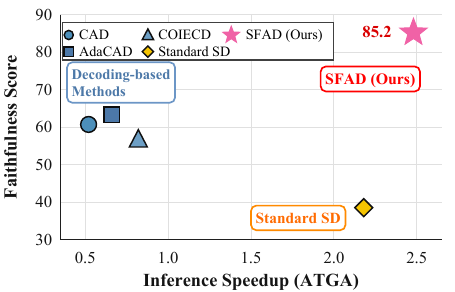}
\caption{Latency-faithfulness trade-off analysis. SFAD achieves Pareto-optimal performance, simultaneously delivering 2.48$\times$ speedup and 85.2 faithfulness score. Decoding-level baselines suffer from severe latency penalties ($<$0.85$\times$) despite moderate faithfulness, while standard SD prioritizes speed (2.18$\times$) at the cost of poor faithfulness (38.5). SFAD outperforms the best baseline by 21.7 points while maintaining 2.9$\times$ faster inference.}
\label{fig:tradeoff}
\end{figure}

\subsection{Latency-Faithfulness Trade-off Analysis}
\label{sec:tradeoff}
A critical question in deploying context-faithful LLMs is whether achieving high factuality necessitates sacrificing inference speed. To address this, we analyze the trade-off between inference acceleration and contextual faithfulness across different decoding strategies. We evaluate all methods on the same 200-instance LLM-AggreFact test set used in Section~\ref{Data_DPO}, measuring both Average Token Generation Acceleration (ATGA, Eq.~\eqref{eq:atga}) and Faithfulness Score (MiniCheck FaithScore). Figure~\ref{fig:tradeoff} visualizes this relationship, revealing three distinct performance regimes.

\textbf{Decoding-level baselines exhibit poor speed-faithfulness balance.} CAD, AdaCAD, and COIECD achieve moderate faithfulness scores but suffer severe latency penalties (ATGA 0.52$\times$--0.82$\times$). They need to perform two forward passes for both contextualized and baseline inputs. This introduces additional token overhead, which slows down the model's inference speed.

\textbf{Standard speculative decoding prioritizes speed over faithfulness.} While achieving 2.18$\times$ speedup, vanilla SD with an unaligned draft model produces drastically poor faithfulness, falling below even the decoding baselines. The draft model's lack of contextual grounding causes it to propose tokens based on parametric priors rather than provided evidence, leading to significant hallucinations in knowledge-intensive scenarios.

\textbf{SFAD achieves Pareto-optimal performance.} SFAD strikes an optimal balance between contextual faithfulness and inference speed, simultaneously delivering 2.48$\times$ speedup and 85.2 faithfulness score.

\subsection{Epistemic Friction Analysis: Detecting Confident Hallucinations}
To validate our Epistemic Friction mechanism, we analyze 1,200 samples from MQUAKE~\cite{zhong-etal-2023-mquake}, tracking $D_{JS}$, $\kappa_t$, and $\mathcal{F}_t$ across generation.
Figure~\ref{fig:epistemic_friction} illustrates metric dynamics across a representative generation sequence. The raw divergence $D_{JS}$ fluctuates frequently, with peaks exceeding $\tau = 0.5$, but relying solely on $D_{JS}$ would trigger excessive false positives at stylistic variations.

\begin{figure}[t]
    \centering
    \includegraphics[width=\linewidth]{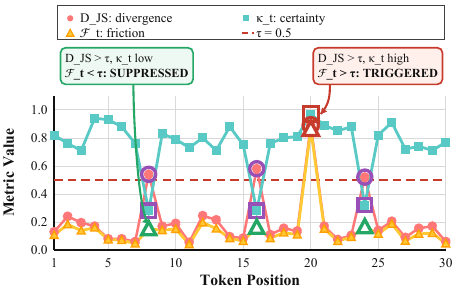}
    \caption{Epistemic Friction dynamics. Purple circles mark positions where $D_{JS} > \tau$ is suppressed by low $\kappa_t$. The red square marks a confident hallucination that triggers steering.}
    \label{fig:epistemic_friction}
\end{figure}

The specialist certainty $\kappa_t$ provides crucial filtering. At positions marked with purple circles, $D_{JS}$ slightly exceeds $\tau$ (0.52--0.58) but $\kappa_t$ remains low (0.28--0.32), indicating uncertainty at linguistic alternatives. Consequently, $\mathcal{F}_t = D_{JS} \times \kappa_t$ is suppressed below threshold (0.15--0.17), correctly avoiding intervention.

In contrast, the red-square position exhibits a confident hallucination: $D_{JS} \approx 0.89$ combined with $\kappa_t \approx 0.97$ produces a sharp spike in $\mathcal{F}_t$ ($\sim$0.86), triggering full logit steering. By conditioning on both distributional conflict and specialist confidence, \textbf{Epistemic Friction provides a principled trigger for factuality correction, avoiding the over-correction pitfalls of divergence-only metrics.}

\subsection{Quantifying Logit Steering Impact}

To evaluate our asymmetric logit steering, we measure the average probability of faithful tokens across three decoding strategies on the MQUAKE test set, where faithful tokens are context-grounded entities that correctly answer the query.

Table~\ref{tab:logit_injection} reveals the fundamental difference between verification-based and steering-based approaches. The original target model assigns 18.73\% probability to faithful tokens, reflecting parametric bias where internal knowledge conflicts with contextual evidence. While the target model considers the faithful token, it remains overshadowed by parametric priors.

Standard speculative decoding provides virtually no correction, with faithful token probability remaining at 18.91\%. This stagnation is inherent to SD's design: the verification mechanism determines whether to accept or reject draft tokens based on $P(\text{accept}) \propto \frac{P_M(x)}{P_m(x)}$, but it does not modify the target model's output distribution. When the draft proposes a faithful token that the target assigns low probability, SD simply rejects it and resamples from the original $P_M$, perpetuating the same parametric bias. This reveals a fundamental limitation: verification-only approaches operate through accept/reject decisions without modifying the target distribution, thus they cannot correct distributional biases but only accelerate generation.

\begin{table}[t]
\centering
\footnotesize
\caption{Average probability of faithful tokens. Standard SD provides negligible improvement, while SFAD achieves substantial gains through logit steering.}
\label{tab:logit_injection}
\begin{tabular}{lccc}
\toprule
\textbf{Strategy} & \textbf{Prob (\%)} & \textbf{Abs. Gain} & \textbf{Rel. Gain} \\
\midrule
Original Target & 18.73 & -- & 1.0× \\
Standard SD & 18.91 & +0.18 & 1.01× \\
\rowcolor[HTML]{D9EAD3}SFAD (Ours) & 62.45 & +43.72 & 3.33× \\
\bottomrule
\end{tabular}
\end{table}

\begin{table}[t]
\centering
\small
\caption{General utility evaluation.}
\label{tab:utility}
\setlength{\tabcolsep}{2.75pt} 
\begin{tabular}{lcccccc} 
\toprule
& \textbf{GSM8K} & \multicolumn{5}{c}{\textbf{Just-Eval (1--5)}} \\
\cmidrule(lr){2-2} \cmidrule(lr){3-7}
\textbf{Method} & \textbf{Acc} & \textbf{Help.} & \textbf{Clar.} & \textbf{Fact.} & \textbf{Depth} & \textbf{Eng.} \\
\midrule
Target (Greedy) & 91.35 & 4.15 & 4.90 & 4.30 & 4.55 & 4.75 \\ 
Standard SD     & 91.32 & 4.14 & 4.89 & 4.28 & 4.54 & 4.74 \\ 
\rowcolor[HTML]{D9EAD3} \textbf{SFAD (Ours)} & 91.27 & 4.11 & 4.88 & 4.33 & 4.53 & 4.73 \\
\bottomrule
\end{tabular}
\end{table}

In contrast, SFAD elevates faithful token probability to 62.45\%, achieving a 3.33× relative improvement and a 243× larger absolute gain compared to standard SD. This dramatic boost results from our asymmetric steering mechanism (Eq.~\eqref{eq:steering}): when epistemic friction $\mathcal{F}_t$ exceeds threshold $\tau$, the system triggers ReLU-based logit injection $z^*_t = z_{M,t} + \lambda_t \cdot \text{ReLU}(z_{m,t} - z_{M,t})$, directly amplifying the draft model's high logits for contextually faithful tokens. Unlike verification that merely accepts or rejects, steering fundamentally reshapes the target's output distribution, transforming faithful tokens from minority candidates (18.73\%) to dominant choices (62.45\%). Critically, this intervention occurs selectively as it only activates when 
both distributional conflict and specialist conviction are present, maintaining 
the 2.48× speedup while achieving 85.2 faithfulness score. \textbf{This analysis demonstrates that logit-level steering provides a principled mechanism for resolving knowledge conflicts, achieving context faithfulness unattainable through verification alone.}

\subsection{SFAD Preserves General Utility}
We evaluate on GSM8K~\cite{cobbe2021training} and Just-Eval~\cite{lin2024unlocking} to verify SFAD maintains general capabilities. As shown in Table~\ref{tab:utility}, SFAD achieves 91.27\% on GSM8K with negligible Just-Eval degradation. This stems from the selective nature of $\mathcal{F}_t$: when no knowledge conflict exists, friction remains below $\tau$ and the system defaults to standard speculative decoding, confirming SFAD's factuality enhancement does not sacrifice general utility or reasoning capability.

\subsection{Generalization and Parameter Analysis}
To evaluate generalization across model families, we apply SFAD with Llama-3.1-8B as the target model. As shown in Table~\ref{tab:foundation_qa_llama}, SFAD consistently improves faithfulness while maintaining inference efficiency. The coefficient $\gamma$ modulates specialist certainty $\kappa_t$ to regulate steering sensitivity. As shown in Table~\ref{tab:gamma_ablation}, lower $\gamma$ increases intervention frequency but risks incorporating unreliable signals, while higher $\gamma$ improves speedup at the cost of faithfulness. 
\begin{table}[t]
\centering
\small
\caption{Llama-3.1-8B results on Foundation QA.}
\label{tab:foundation_qa_llama}
\resizebox{\columnwidth}{!}{%
\begin{tabular}{lcccc}
\toprule
\textbf{Method} & \textbf{TriviaQA} & \textbf{HotpotQA} & \textbf{PopQA} & \textbf{Rel. Latency} \\
\midrule
\multicolumn{5}{l}{\textit{Vanilla Baseline}} \\
Llama-3.1-8B & 49.23 & 38.45 & 75.18 & 1.00$\times$ \\
\midrule
\multicolumn{5}{l}{\textit{Decoding-level Baselines (Llama-3.1-8B)}} \\
CAD & 37.82 & 36.21 & 68.45 & 2.00$\times$ \\
AdaCAD & 78.56 & 42.31 & 74.28 & 2.18$\times$ \\
COIECD & \underline{79.43} & \underline{42.88} & \underline{75.56} & 2.35$\times$ \\
\midrule
\rowcolor[HTML]{D9EAD3}
\textbf{SFAD (Ours)} & \textbf{81.67} & \textbf{48.73} & \textbf{83.21} & \textbf{0.85$\times$} \\
\bottomrule
\end{tabular}%
}
\end{table}

\begin{table}[t]
\centering
\small
\setlength{\abovecaptionskip}{3pt}
\setlength{\belowcaptionskip}{5pt}
\caption{Ablation of sharpening coefficient $\gamma$.}
\label{tab:gamma_ablation}
\setlength{\tabcolsep}{2.0pt}
\begin{tabular}{ccccc}
\toprule
$\boldsymbol{\gamma}$ & \textbf{Steer (\%)} & \textbf{Faith.} ($\uparrow$) & \textbf{Speedup} ($\uparrow$) & \textbf{R-L} ($\uparrow$) \\
\midrule
1 & 31.2 & 83.7 & 2.31$\times$ & 15.89 \\
\rowcolor[HTML]{D9EAD3} \textbf{2 (Default)} & \textbf{22.4} & \textbf{85.2} & 2.48$\times$ & \textbf{16.32} \\
3 & 15.8 & 84.1 & 2.61$\times$ & 16.18 \\
4 & 9.3 & 78.5 & \textbf{2.74$\times$} & 15.94 \\
\bottomrule
\end{tabular}
\end{table}

%% file: Section/7_Conclusion.tex
\section{Conclusion}
In this work, we propose \textbf{SFAD}, the first speculative decoding framework unifying contextual faithfulness with acceleration. Leveraging a context-faithful drafter as factuality sentinel, SFAD detects hallucinations via epistemic friction. Drafters are trained on ConFide, a dataset with atomic-level hallucination perturbations. At inference, SFAD selectively steers target logits upon conflict detection while preserving efficiency, achieving simultaneous speedup and faithfulness gains.

%% file: Section/8_impact_statement.tex
\section*{Limitations}
SFAD requires a domain-aligned draft model trained via DPO, which introduces additional data construction overhead compared to standard speculative decoding. Furthermore, the friction threshold $\tau$ may require tuning when applied to out-of-distribution domains. 

%% file: Section/Ack.tex
\section*{Acknowledgment}
Di Wang is supported in part by the funding BAS/1/1689-01-01, RGC/3/7125-01-01, FCC/1/5940-20-05, FCC/1/5940-06-02, and King Abdullah University of Science and Technology (KAUST) – Center of Excellence for Generative AI, under award number 5940 and a gift from Google.

\noindent
Lijie Hu is supported in part by the funding BF0100.

%% file: Section/appendix_two.tex
\twocolumn
\raggedbottom

\section{Parameter Analysis}
\label{parameter}

\subsection{Role of Contextual Plausibility Mask (\texorpdfstring{$\eta$}{eta})}
\label{subsec:eta_analysis}

A pivotal challenge in logit-level steering is the potential for linguistic degradation: a specialist may propose a factually accurate token that is syntactically incompatible with the target model's prefix. To evaluate how our \textbf{Contextual Plausibility Mask (CPM)} mitigates this risk, we conduct an ablation study on the threshold $\eta$ using the XSum dataset.

\begin{table}[H]
\centering
\small
\caption{Ablation of the Plausibility Threshold $\eta$ on XSum. $\eta=0.1$ is the default setting used in Table \ref{tab:summarization}; $\eta=0$ denotes steering without the linguistic safety guard.}
\label{tab:eta_ablation}
\resizebox{\columnwidth}{!}{%
\begin{tabular}{lccc}
\toprule
\textbf{Configuration} & \textbf{ROUGE-L} ($\uparrow$) & \textbf{BERT-P} ($\uparrow$) & \textbf{AlignScore} ($\uparrow$) \\
\midrule
SFAD (Standard, $\eta=0.1$) & \underline{16.32} & \textbf{93.97} & 87.37 \\
SFAD w/o CPM ($\eta=0$) & 14.78 & 91.25 & \textbf{87.49} \\
Strict CPM ($\eta=0.5$) & 15.42 & 92.84 & 81.12 \\
\midrule
\textit{Vanilla Qwen3-14B} & 13.67 & 91.67 & 72.86 \\
\bottomrule
\end{tabular}%
}
\end{table}

As illustrated in Table \ref{tab:eta_ablation}, removing the mask ($\eta=0$) yields the highest \textit{AlignScore} (87.49), confirming that the specialist model is indeed capable of forcing factual corrections. However, this comes at a substantial cost to linguistic integrity: \textbf{ROUGE-L drops by 1.54 points} and \textbf{BERT-P falls by 2.72 points} compared to the standard SFAD. Qualitative analysis reveals that without CPM, the model often injects correct entities with mismatched syntax (e.g., incorrect prepositional usage or broken noun phrases).

By introducing $\eta=0.1$, SFAD successfully filters out these "linguistically dissonant" corrections. This explains the competitive results in Table \ref{tab:summarization}, where SFAD nearly matches the 70B Frontier Model's performance in both fluency and factuality. The CPM acts as a "safety guard," ensuring that steering only occurs within the target model's plausible manifold, thus achieving a superior Pareto-optimal balance between knowledge correction and natural language generation.

\subsection{Sensitivity Analysis of Friction Threshold \texorpdfstring{$\tau$}{tau}}
\label{subsec:tau_sensitivity}

The friction threshold $\tau$ serves as the primary hyperparameter for calibrating the sensitivity of SFAD's hallucination detector. It determines the operating point on the factuality-efficiency Pareto frontier by controlling the ratio between the \textit{Steering Path} and the \textit{Fast Path}. We evaluate the impact of $\tau$ on the LLM-AggreFact test set, reporting the Steering Ratio (percentage of tokens where $F_t \geq \tau$), Faithfulness Score (MiniCheck), and inference speedup (ATGA).

\begin{table}[H]
\centering
\small
\caption{Sensitivity analysis of $\tau$ on LLM-AggreFact. \textit{Steering Ratio} denotes the percentage of generated tokens processed via the logit steering path. Speedup is relative to greedy decoding without speculation.}
\label{tab:tau_sensitivity}
\resizebox{\columnwidth}{!}{%
\begin{tabular}{lccc}
\toprule
\textbf{Threshold $\tau$} & \textbf{Steering Ratio (\%)} & \textbf{Faithfulness Score} ($\uparrow$) & \textbf{Speedup (ATGA)} ($\uparrow$) \\
\midrule
$\tau = 0.1$ (Aggressive) & 48.2\% & \textbf{86.7} & 2.12$\times$ \\
$\tau = 0.3$              & 35.6\% & 85.9 & 2.31$\times$ \\
\rowcolor[HTML]{D9EAD3}
$\tau = 0.5$ (Default)    & 22.4\% & \underline{85.2} & \underline{2.48$\times$} \\
$\tau = 0.7$              & 12.1\% & 62.4 & 2.65$\times$ \\
$\tau = 0.9$ (Conservative) & 4.5\%  & 41.8 & \textbf{2.82$\times$} \\
\midrule
\textit{Standard SD} & 0.0\% & 38.5 & 2.18$\times$ \\
\bottomrule
\end{tabular}%
}
\end{table}

As shown in Table \ref{tab:tau_sensitivity}, the \textit{Steering Ratio} decreases monotonically as $\tau$ increases, reflecting a more selective intervention strategy. When $\tau=0.1$, the model intervenes frequently, achieving the highest faithfulness but incurring a latency penalty due to frequent target model logit computations. Conversely, at $\tau=0.9$, the system defaults to standard speculative decoding (Fast Path) most of the time, maximizing speed but failing to resolve knowledge conflicts.

Crucially, our default setting of $\tau=0.5$ achieves a "sweet spot": it resolves the majority of hallucinations (surpassing standard SD by 46.7 points) while maintaining a high speedup of 2.48$\times$. This demonstrates that SFAD performs \textbf{surgical interventions} only when necessary.

\subsection{Effectiveness of Asymmetric Logit Fusion Operators}
\label{app:fusion_ablation}

To empirically validate the theoretical advantage of our asymmetric steering law (Theorem~\ref{thm:asymmetric-switching}), we conduct an ablation study comparing different logit fusion operators. In this experiment, we keep the adaptive gating mechanism ($\lambda_t$) and the Contextual Plausibility Mask (CPM) constant, only varying how the specialist's logits $z_m$ are integrated with the generalist's logits $z_M$ when steering is triggered. We evaluate four strategies:
\begin{enumerate}[leftmargin=*,label=(\roman*)]
    \item \textbf{Linear Sum:} $z^* = z_M + \lambda_t z_m$, which directly adds the specialist's signals.
    \item \textbf{Linear Interpolation:} $z^* = (1-\lambda_t)z_M + \lambda_t z_m$, a standard weighted average used in model blending.
    \item \textbf{Subtractive Contrast:} $z^* = z_M + \lambda_t (z_M - z_{base})$, similar to contrastive decoding but using the specialist as the positive signal.
    \item \textbf{Asymmetric Steering (SFAD):} $z^* = z_M + \lambda_t \text{ReLU}(z_m - z_M)$, our proposed unidirectional injection.
\end{enumerate}

The results on PopQA and HotpotQA are summarized in Table~\ref{tab:fusion_ablation}.

\begin{table}[H]
\centering
\small
\caption{Ablation of Logit Fusion Operators. Faithfulness is measured by Exact Match (EM) for PopQA and MiniCheck for HotpotQA. Fluency is represented by ROUGE-L (R-L).}
\label{tab:fusion_ablation}
\resizebox{\columnwidth}{!}{%
\begin{tabular}{lcccc}
\toprule
\textbf{Fusion Strategy} & \multicolumn{2}{c}{\textbf{PopQA}} & \multicolumn{2}{c}{\textbf{HotpotQA}} \\
\cmidrule(lr){2-3} \cmidrule(lr){4-5}
& \textbf{EM} ($\uparrow$) & \textbf{R-L} ($\uparrow$) & \textbf{Faith} ($\uparrow$) & \textbf{R-L} ($\uparrow$) \\
\midrule
Linear Sum           & 78.42 & 14.12 & 46.21 & 35.89 \\
Linear Interpolation  & 81.35 & 15.67 & 48.55 & 39.42 \\
Subtractive Contrast & 83.12 & 13.05 & 49.32 & 34.11 \\
\midrule
\rowcolor[HTML]{D9EAD3}\textbf{SFAD (Ours)} & \textbf{86.39} & \textbf{16.32} & \textbf{52.19} & \textbf{41.34} \\
\bottomrule
\end{tabular}%
}
\end{table}

As shown in Table~\ref{tab:fusion_ablation}, the \textbf{Asymmetric Steering} operator consistently outperforms other fusion methods. \textit{Linear Sum} and \textit{Interpolation} tend to dilute the generalist's linguistic priors, leading to a drop in ROUGE-L (fluency). Conversely, while \textit{Subtractive Contrast} can highlight factual differences, it often results in "negative constraints" that suppress even valid tokens, harming coherence. SFAD's ReLU-based injection ensures that the specialist only intervenes to \textit{boost} tokens where it has higher confidence than the generalist, preserving the natural language manifold of the target model while effectively correcting factual errors.

\section{Generalization Across Model Families}
\label{app:model-family}

We provide more results on Llama-3.1-8B in Table \ref{tab:summarization_llama} and Table \ref{tab:longform_qa_llama}.

\begin{table}[H]
\centering
\small
\caption{Llama-3.1-8B results on Summarization.}
\label{tab:summarization_llama}
\resizebox{\columnwidth}{!}{%
\begin{tabular}{lccccc}
\toprule
& \multicolumn{3}{c}{\textbf{XSum}} & \textbf{TofuEval} & \\
\cmidrule(lr){2-4} \cmidrule(lr){5-5}
\textbf{Method} & \textbf{R-L} & \textbf{BERT-P} & \textbf{AlignScore} & \textbf{AlignScore} & \textbf{Rel. Latency} \\
\midrule
\multicolumn{6}{l}{\textit{Vanilla Baseline}} \\
Llama-3.1-8B & 12.89 & 90.83 & 71.24 & 58.12 & 1.00$\times$ \\
\midrule
\multicolumn{6}{l}{\textit{Decoding-level Baselines (Llama-3.1-8B)}} \\
CAD & 13.76 & 92.87 & 82.56 & 81.45 & 2.00$\times$ \\
AdaCAD & 14.08 & \underline{93.41} & \underline{84.03} & \underline{83.29} & 2.22$\times$ \\
COIECD & 12.91 & 90.22 & 72.15 & 59.34 & 2.48$\times$ \\
\midrule
\rowcolor[HTML]{D9EAD3}
\textbf{SFAD (Ours)} & \textbf{15.47} & \textbf{93.12} & \textbf{85.68} & \textbf{85.87} & \textbf{0.87$\times$} \\
\bottomrule
\end{tabular}%
}
\end{table}

\begin{table}[H]
\centering
\small
\setlength{\tabcolsep}{4pt}
\caption{Llama-3.1-8B results on Long-Form QA.}
\label{tab:longform_qa_llama}
\resizebox{\columnwidth}{!}{%
\begin{tabular}{lcccccc|c}
\toprule
& \multicolumn{2}{c}{\textbf{CLAPNQ}} & \multicolumn{2}{c}{\textbf{ExpertQA}} & \multicolumn{2}{c}{\textbf{HAGRID}} & \\
\cmidrule(lr){2-3} \cmidrule(lr){4-5} \cmidrule(lr){6-7}
\textbf{Method} & \textbf{R-L} & \textbf{Faith} & \textbf{R-L} & \textbf{Faith} & \textbf{R-L} & \textbf{Faith} & \textbf{Rel. Latency} \\
\midrule
\multicolumn{8}{l}{\textit{Vanilla Baseline}} \\
Llama-3.1-8B & 16.34 & 57.89 & 30.21 & 49.56 & 16.12 & 55.87 & 1.00$\times$ \\
\midrule
\multicolumn{8}{l}{\textit{Decoding-level Baselines (Llama-3.1-8B)}} \\
CAD & 17.41 & 58.43 & 32.15 & 51.72 & 17.03 & 56.45 & 2.00$\times$ \\
AdaCAD & 17.58 & 60.51 & 31.49 & 52.93 & 16.31 & 58.12 & 2.21$\times$ \\
COIECD & \underline{17.72} & \underline{60.14} & \underline{32.78} & \underline{54.48} & \underline{16.52} & \underline{57.94} & 2.45$\times$ \\
\midrule
\rowcolor[HTML]{D9EAD3}
\textbf{SFAD (Ours)} & \textbf{19.87} & \textbf{68.45} & \textbf{34.56} & \textbf{62.37} & \textbf{18.93} & \textbf{66.23} & \textbf{0.81$\times$} \\
\bottomrule
\end{tabular}%
}
\end{table}

These cross-family results indicate that SFAD's gains are not tied to a single target backbone. Across summarization and long-form QA, the same decoding mechanism improves faithfulness while preserving the intended latency advantage, suggesting that the learned factuality-aware draft model and friction-triggered steering provide a portable correction signal across model families.

\clearpage
\twocolumn
\section{Examples of Factual Perturbation for ConFide}

We present two representative cases of how a faithful response is transformed into a hallucinated negative sample ($y_l$) through our atomic-level perturbation mechanism.

\begin{promptbox}{Case 1: Entity Swap ($\tau_{ent}$)}
    \footnotesize
    \sloppy
    \textbf{Source Context:} \\
    ``...Muto graduated from Keio University in Tokyo with an economics degree two weeks ago. He will join Chelsea's partner club Vitesse Arnhem on loan...'' 
    
    \vspace{0.8em}
    
    \textbf{Original Faithful Response ($y_w$):} \\
    ``Yoshinori Muto, who recently completed his economics degree at Keio University, is set to join Chelsea.'' 
    
    \vspace{0.8em}
    
    \textbf{Atomic Fact Decomposition:} \\
    $a_j = (\text{Yoshinori Muto}, \text{graduated from}, \text{\textbf{Keio University}})$ 
    
    \vspace{0.8em}
    
    \textbf{Perturbation Operation:} \\
    Swap target entity \textit{Keio University} with intra-context distractor \textit{\textbf{Vitesse Arnhem}}. 
    
    \vspace{0.8em}
    
    \textbf{Generated Hallucinated Response ($y_l$):} \\
    ``Yoshinori Muto, who recently completed his economics degree at \colorbox{red!20}{Vitesse Arnhem}, is set to join Chelsea.''

\end{promptbox}

\vspace{1em} 

\begin{promptbox}{Case 2: Relation Inversion ($\tau_{rel}$)}
    \footnotesize
    \sloppy
    \textbf{Source Context:} \\
    ``...Ogane claims that Chelsea's interest in Muto is not connected to the £200million sponsorship deal they signed with Yokohama Rubber...'' 
    
    \vspace{0.8em}
    
    \textbf{Original Faithful Response ($y_w$):} \\
    ``The club president stated that the move for Muto is independent of the sponsorship deal with Yokohama Rubber.'' 
    
    \vspace{0.8em}
    
    \textbf{Atomic Fact Decomposition:} \\
    \[
    \begin{aligned}
    a_j = (&\text{Chelsea's interest},\ \text{\textbf{is not connected to}},\\
           &\text{Sponsorship Deal})
    \end{aligned}
    \]
    
    \vspace{0.8em}
    
    \textbf{Perturbation Operation:} \\
    Apply logical inversion ($\neg r_j$) to the relation \textit{is not connected to}. 
    
    \vspace{0.8em}
    
    \textbf{Generated Hallucinated Response ($y_l$):} \\
    ``The club president stated that the move for Muto \colorbox{red!20}{is a direct result of} the sponsorship deal with Yokohama Rubber.''

\end{promptbox}

\par\vspace{1em}
\section{Algorithms}
\label{app:algorithms}

This section provides the algorithmic details of the proposed framework.
Algorithm~\ref{alg:sfad_process} summarizes the inference-time procedure of SFAD,
including friction-based conflict detection and asymmetric logit steering.
Algorithm~\ref{alg:confide_construction} describes the construction process of
the ConFide preference dataset used for training the factuality-aware draft model.

\begin{algorithm}[H]
\footnotesize
\sloppy
   \caption{SFAD: Speculative Factuality-Aware Decoding}
   \label{alg:sfad_process}
\setlength{\algorithmicindent}{0.75em}
\begin{algorithmic}[1]
   \STATE {\bfseries Input:} Target model $M$, DPO-aligned draft model $m$, prefix context $x_{<t}$, friction threshold $\tau$, sharpening coefficient $\gamma$, plausibility threshold $\eta$, sigmoid scale $\beta$.
   \STATE {\bfseries Output:} Context-faithful generated token $x_t$.
   
   \WHILE{$x_t \neq \text{EOS}$}
      \STATE // \textit{Step 1: Distribution Computation}
      \STATE Compute logit vectors $\mathbf{z}_{M,t}$ from $M(x_{<t})$ and $\mathbf{z}_{m,t}$ from $m(x_{<t})$
      \STATE Compute probabilities $\mathbb{P}_{M} = \text{Softmax}(\mathbf{z}_{M,t})$ and $\mathbb{P}_{m} = \text{Softmax}(\mathbf{z}_{m,t})$
      
      \STATE // \textit{Step 2: Factuality Conflict Detection}
      \STATE Calculate Specialist Certainty: $\kappa_t = \left( 1 - \frac{-\sum_{v \in \mathcal{V}} \mathbb{P}_m(v) \log \mathbb{P}_m(v)}{\log |\mathcal{V}|} \right)^\gamma$
      \STATE Calculate Epistemic Friction: $\mathcal{F}_t = D_{JS}(\mathbb{P}_M \parallel \mathbb{P}_m) \cdot \kappa_t$
      
      \IF{$\mathcal{F}_t \geq \tau$}
         \STATE // \textit{Steering Path: Distribution Correction}
         \STATE Identify plausible set: $\mathcal{V}_{CPC} = \{x \in \mathcal{V} : \mathbb{P}_m(x|x_{<t}) \geq \eta \cdot \max_{w \in \mathcal{V}} \mathbb{P}_m(w|x_{<t})\}$
         \STATE Compute gating scalar: $\lambda_t = \frac{1}{1 + \exp(-\beta(\mathcal{F}_t - \tau))}$
         \STATE Apply Asymmetric Logit Steering: 
         \STATE $\mathbf{z}^*_t = \mathbf{z}_{M,t} + \lambda_t \cdot \max(0, \mathbf{z}_{m,t} - \mathbf{z}_{M,t}) \cdot \mathbb{I}(x \in \mathcal{V}_{CPC})$
         \STATE Sample next token: $x_t \sim \text{Softmax}(\mathbf{z}^*_t)$
      \ELSE
         \STATE // \textit{Fast Path: Standard Speculative Decoding}
         \STATE Sample draft token $\tilde{x}_t$ from $\mathbb{P}_m$
         \STATE $x_t \leftarrow \text{Verify}(\tilde{x}_t, \mathbb{P}_M)$ \COMMENT{Standard rejection sampling}
      \ENDIF
      
      \STATE Update sequence: $x_{<t+1} \leftarrow [x_{<t}; x_t]$
   \ENDWHILE
\end{algorithmic}
\end{algorithm}

\begin{algorithm}[H]
\footnotesize
\sloppy
    \caption{ConFide: Factuality-Aware Preference Data Construction}
    \label{alg:confide_construction}
    \setlength{\algorithmicindent}{0.75em}
\begin{algorithmic}[1]
        \REQUIRE $\mathcal{D} = \{(x, y_{ref}, l)\}_{i=1}^N$, Teacher $\mathcal{M}_{T}$, Paraphraser $\pi_{para}$, Generator $f_{rec}$
        \STATE $\mathcal{P} \leftarrow \emptyset$
        \FOR{each $(x, y_{ref}, l) \in \mathcal{D}$}
            \STATE $\mathcal{A} = \{a_1, \dots, a_m\} \leftarrow f_{dec}(y_{ref})$ \COMMENT{Atomic Fact Decomposition}
            
            \STATE // \textit{Negative Generation: Perturb $a_j \in \mathcal{A}$ via Entity Swap, Numerical, or Relation Inversion}
            \STATE $a'_j \leftarrow \phi(a_j), \quad y_l \leftarrow f_{rec}(\mathcal{A} \setminus \{a_j\} \cup \{a'_j\})$ \COMMENT{Corrupt and Reconstruct}
            
            \STATE // \textit{Positive Refinement: Paraphrase faithful samples or utilize Teacher for corrections}
            \STATE $y_w \leftarrow (l = 1) \ ? \ \pi_{para}(y_{ref}) \ : \ \mathcal{M}_{T}(\text{correct } y_{ref} \text{ based on } x)$
            
            \STATE $\mathcal{P} \leftarrow \mathcal{P} \cup \{(x, y_w, y_l)\}$ \COMMENT{Assemble Preference Pair}
        \ENDFOR
        \STATE \textbf{return} $\mathcal{P}$ 
    \end{algorithmic}
\end{algorithm}
\clearpage
\twocolumn

\section{Theoretical Analysis of SFAD}
\label{app:theory}

In this section, we provide a rigorous formal justification for the \textit{Speculative Factuality-Aware Decoding} (SFAD) framework. Our analysis bridges the gap between empirical performance and distributional theory by focusing on three core pillars: (1) \textbf{Semantic Integrity}, establishing that our steering mechanism preserves the linguistic manifold; (2) \textbf{Factuality Amplification}, proving how DPO-induced margins yield exponential gains in faithful tokens; and (3) \textbf{Dynamic Stability}, demonstrating the robustness of asymmetric steering against the "zero-probability trap" prevalent in contrastive methods.

\subsection{Semantic Consistency and the Linguistic Manifold Bound}
\label{app:manifold_bound}

A fundamental challenge in logit steering is the \textit{alignment-fidelity trade-off}: factual corrections must not drive the model's output distribution off the natural language manifold $\mathcal{M}$. We formalize the role of the Contextual Plausibility Mask (CPM) as a projection operator that constrains the steering signal to the support of the generalist's distribution.

\begin{theorem}[Manifold Projection Bound]
\label{thm:manifold_projection}
Let $\mathcal{V}_{CPC} = \{v \in \mathcal{V} : P_m(v | x_{<t}) \geq \eta \cdot \max_w P_m(w | x_{<t})\}$ be the set of plausible candidates. For any steering intensity $\lambda_t \in [0,1]$, the Total Variation (TV) distance between the steered distribution $P^*$ and the original generalist distribution $P_M$ is bounded by the probability mass concentrated on the plausible manifold:
\begin{equation}
\begin{aligned}
 d_{TV}(P^*, P_M)
 &\leq \frac{1}{2}\sum_{v\in\mathcal{V}_{CPC}} P_M(v)
 \left|\frac{e^{\lambda_t\Delta z_v}}{\mathcal{Z}_t^*}-1\right| \\
 &\quad+\frac{1}{2}\left|1-\frac{1}{\mathcal{Z}_t^*}\right|
 (1-P_M(\mathcal{V}_{CPC})) .
\end{aligned}
\end{equation}
where $\Delta z_v = (z_{m,v} - z_{M,v})_+$ and $\mathcal{Z}_t^* = \mathbb{E}_{v \sim P_M} [e^{\lambda_t \Delta z_v}]$ is the partition normalization factor.
\end{theorem}

\begin{proof}
By the definition of the CPM-based steering in Eq.~\eqref{eq:steering}, for any token $v \notin \mathcal{V}_{CPC}$, the steering signal $\Delta z_v$ is nullified by the indicator function $\mathbb{I}(x \in \mathcal{V}_{CPC})$. Consequently, the steered logit $z^*_v$ equals the original logit $z_{M,v}$, and the probability simplifies to $P^*(v) = P_M(v) / \mathcal{Z}_t^*$. 

The TV distance is defined as $d_{TV}(P^*, P_M) = \frac{1}{2} \sum_{v \in \mathcal{V}} |P^*(v) - P_M(v)|$. Partitioning the vocabulary $\mathcal{V}$ into the plausible set $\mathcal{V}_{CPC}$ and its complement $\mathcal{V}^c$, we obtain:
\begin{align}
2d_{TV} &= \sum_{v\in\mathcal{V}_{CPC}}
 \left|\frac{P_M(v)e^{\lambda_t\Delta z_v}}{\mathcal{Z}_t^*}-P_M(v)\right| \nonumber\\
&\quad+\sum_{v\in\mathcal{V}^c}
 \left|\frac{P_M(v)}{\mathcal{Z}_t^*}-P_M(v)\right| \nonumber\\
&=\sum_{v\in\mathcal{V}_{CPC}}P_M(v)
 \left|\frac{e^{\lambda_t\Delta z_v}}{\mathcal{Z}_t^*}-1\right| \nonumber\\
&\quad+\left|\frac{1}{\mathcal{Z}_t^*}-1\right|
 \sum_{v\in\mathcal{V}^c}P_M(v).
\end{align}
Substituting $P_M(\mathcal{V}^c) = 1 - P_M(\mathcal{V}_{CPC})$ completes the proof. This bound illustrates that the divergence is strictly governed by the specialist's corrective signal magnitude within the linguistic manifold. Since $m$ and $M$ share a common pre-training ancestry, their high-density regions in $\mathcal{M}$ are naturally aligned, ensuring that steering only re-allocates mass within valid semantic clusters.
\end{proof}

\begin{remark}
\textbf{The Anchor Property.} Theorem \ref{thm:manifold_projection} reveals that SFAD acts as a "factual re-ranker" rather than a "stochastic generator." By anchoring the correction to the generalist's manifold $\mathcal{M}$, the framework prevents out-of-distribution (OOD) artifacts. Unlike additive noise or unconstrained logit shifts, the CPM ensures that the model never samples tokens that are linguistically nonsensical, explaining the high fluency scores observed in our experiments.
\end{remark}

\subsection{Factuality Amplification via DPO-Induced Margins}
\label{app:dpo_amplification}

While the previous section established safety, we now prove that SFAD effectively resolves knowledge conflicts by leveraging the preference gap instilled during the Direct Preference Optimization (DPO) phase.

\begin{proposition}[Exponential Posterior Gain]
\label{prop:exponential_gain}
Let $v_{ctxt}$ be a faithful token and $v_{mem}$ be a hallucinated token (favored by the target model's parametric prior). If the specialist $m$ satisfies the DPO optimality condition with a learned margin $\gamma_m = z_m(v_{ctxt}) - z_m(v_{mem})$, the posterior odds ratio under SFAD steering satisfies:
\begin{equation}
    \frac{P^*(v_{ctxt})}{P^*(v_{mem})} = \frac{P_M(v_{ctxt})}{P_M(v_{mem})} \cdot \exp\left( \lambda_t \cdot \gamma_{eff} \right)
\end{equation}
where $\gamma_{eff} = \max(0, z_{m,v_{ctxt}} - z_{M,v_{ctxt}})$ is the effective corrective margin.
\end{proposition}

\begin{proof}
Based on the Bradley-Terry model utilized in DPO, the specialist $m$ is trained to maximize the log-odds of contextually loyal pairs. Following the asymmetric steering law $z_t^*(v) = z_{M,t}(v) + \lambda_t \Delta z_v$, we examine the log-odds ratio:
\begin{align}
\log \frac{P^*(v_{ctxt})}{P^*(v_{mem})}
&= \bigl(z_{M,v_{ctxt}}+\lambda_t\Delta z_{v_{ctxt}}\bigr) \nonumber\\
&\quad-\bigl(z_{M,v_{mem}}+\lambda_t\Delta z_{v_{mem}}\bigr).
\end{align}
In a factuality conflict, $v_{mem}$ is the "surface pattern" favored by $M$, thus $z_{m,v_{mem}} \leq z_{M,v_{mem}}$, which implies $\Delta z_{v_{mem}} = 0$ due to the ReLU activation. Conversely, for the faithful token $v_{ctxt}$, the specialist (having been optimized via DPO) yields $z_{m,v_{ctxt}} > z_{M,v_{ctxt}}$. Thus:
\begin{align}
\log \frac{P^*(v_{ctxt})}{P^*(v_{mem})}
&= \log \frac{P_M(v_{ctxt})}{P_M(v_{mem})} \nonumber\\
&\quad+\lambda_t\bigl(z_{m,v_{ctxt}}-z_{M,v_{ctxt}}\bigr).
\end{align}
Exponentiating both sides demonstrates that the faithful candidate's probability mass is amplified exponentially relative to the hallucination, controlled by the steering intensity $\lambda_t$.
\end{proof}

\begin{remark}
\textbf{Selective Pressure.} This result highlights the advantage of SFAD over standard speculative decoding. While the latter performs binary rejection (0 or 1), SFAD applies continuous \textit{selective pressure}. Even if the generalist $M$ is initially biased toward a hallucination, a sufficiently certain specialist can shift the distribution's mode toward the truth without requiring a complete rejection of the sequence.
\end{remark}

\subsection{Information-Theoretic Validity of Epistemic Friction}

We now justify the use of Epistemic Friction $F_t$ as a dynamic trigger. It must distinguish between \textit{benign diversity} and \textit{factual conflict}.

\begin{lemma}[Friction as a High-Precision Trigger]
\label{lemma:friction_validity}
The Epistemic Friction $F_t = \mathcal{D}_{JS}(P_M \| P_m) \cdot \kappa_t$ is a lower bound on the expected reduction in epistemic uncertainty when the generalist conditions its output on the specialist's contextual evidence.
\end{lemma}

\begin{proof}
The Jensen-Shannon divergence $\mathcal{D}_{JS}$ is a symmetric and bounded metric of distributional tension. In the context of LLMs, high divergence can arise from two sources: (i) factual disagreement or (ii) high entropy (uncertainty) in one model. 
Standard entropy-based triggers fail because they cannot distinguish between these cases. By weighting $\mathcal{D}_{JS}$ with the specialist certainty $\kappa_t = (1 - H(P_m)/\log |\mathcal{V}|)^\gamma$, $F_t$ acts as a filter. As $\kappa_t \to 1$ (certainty), $F_t \to \mathcal{D}_{JS}$. As $\kappa_t \to 0$ (uncertainty), $F_t \to 0$, suppressing the trigger. Thus, $F_t$ only activates steering when the specialist's corrective signal is both \textit{strong} and \textit{high-confidence}, minimizing noise injection.
\end{proof}

\subsection{Stability and Support Preservation}

A critical failure in subtractive contrastive methods (e.g., CAD) is the "Zero-Probability Trap," where valid tokens are suppressed to numerical underflow. 

\begin{theorem}[Numerical Stability and Support Preservation]
\label{thm:stability}
For any finite steering scale $\lambda_t \in [0, \infty)$ and any token $v \in \mathcal{V}$, the SFAD framework satisfies the support preservation property: if $P_M(v) > 0$, then $P^*(v) > 0$.
\end{theorem}

\begin{proof}
The steered logit is $z^*_v = z_{M,v} + \lambda_t \cdot \text{ReLU}(z_{m,v} - z_{M,v})$. Since $\text{ReLU}(\cdot) \geq 0$ and $\lambda_t \geq 0$, it holds that $z^*_v \geq z_{M,v}$ for all $v$. 
The steered probability is $P^*(v) = \exp(z^*_v) / \sum_w \exp(z^*_w)$. 
Since the exponential function is strictly positive, and $z^*_v$ is bounded from below by the original logit $z_{M,v}$, the numerator remains positive. Unlike subtractive methods where $z^* = z_M - \alpha z_{neg}$ can lead to $-\infty$ and numerical collapse, our asymmetric law ensures that no token is ever strictly "killed." They are only relatively de-emphasized by the factual amplification of superior candidates.
\end{proof}

\begin{remark}
\textbf{Robustness by Design.} This theorem explains why SFAD is robust to the hyperparameter $\lambda_t$. The additive, ReLU-based nature of the steering ensures that even with high steering intensity, the model maintains its linguistic foundation, avoiding the "repetitive gibberish" or "empty strings" common in negative-constraint decoding.
\end{remark}

\subsection{Mode-Switching Optimality under Knowledge Conflicts}

Finally, we demonstrate why asymmetric steering is superior to linear logit interpolation (averaging).

\begin{theorem}[Optimality of Asymmetric Mode-Switching]
\label{thm:asymmetric-switching}
In a bimodal conflict between a factual token $v_f$ and a hallucinated token $v_h$, the asymmetric law $z^* = z_M + \lambda_t (z_m - z_M)_+$ is a rank-preserving transformation for all neutral tokens $v_n$ where $z_{m,v_n} \leq z_{M,v_n}$.
\end{theorem}

\begin{proof}
Consider two neutral tokens $v_1, v_2$ where the generalist and specialist are in relative agreement (i.e., $z_m \leq z_M$ for both). Under SFAD, $\Delta z_{v_1} = \Delta z_{v_2} = 0$. Therefore, their steered logits remain $z^*_{v_1} = z_{M,v_1}$ and $z^*_{v_2} = z_{M,v_2}$. Their relative rank $z^*_{v_1} - z^*_{v_2} = z_{M,v_1} - z_{M,v_2}$ is perfectly preserved. 
In contrast, a linear interpolation $z_{interp} = (1-\alpha)z_M + \alpha z_m$ would modify both logits, potentially flipping their rank if the specialist has slight fluctuations in its low-probability tail. SFAD thus provides a "surgical" correction: it only perturbs the distribution where the specialist has a \textit{strictly better} (more factual) proposal.
\end{proof}

\begin{remark}
\textbf{Theoretical Scalability.} This result suggests that as the Specialist model $m$ improves in factuality (e.g., through more rigorous DPO), SFAD's performance will scale without requiring a corresponding increase in the Generalist's capacity, making it a sustainable architecture for deploying large-scale factual models.
\end{remark}

\subsection{Optimal Risk-Aware Switching and the Factuality-Efficiency Frontier}
\label{app:optimal_switching}

A distinctive feature of SFAD is the hybrid policy (Eq.~\eqref{eq:hybrid}) that dynamically chooses between the \textit{Steering Path} and the \textit{Fast Path}. We now provide a decision-theoretic justification for this switching logic, demonstrating that the Epistemic Friction $F_t$ serves as an optimal proxy for balancing factual integrity and computational latency.

\begin{theorem}[Factuality-Risk Minimization]
\label{thm:risk_min}
Let $\mathcal{R}(\pi)$ be the expected risk of generating a hallucinated token under policy $\pi$, and let $\mathcal{C}(\pi)$ be the computational cost (latency). Define the factual risk $r_t$ at step $t$ as the probability that the draft token $\tilde{x}_t$ deviates from the specialist’s high-confidence manifold. The SFAD hybrid policy $\pi_{SFAD}$ is a solution to the constrained optimization problem:
\begin{equation}
    \min_{\pi} \mathbb{E}[\mathcal{R}(\pi)] \quad \text{subject to} \quad \mathbb{E}[\mathcal{C}(\pi)] \leq \mathcal{B}
\end{equation}
where the switching threshold $\tau$ acts as the Lagrange multiplier $\beta^{-1}$ that determines the operating point on the Factuality-Efficiency Pareto frontier.
\end{theorem}

\begin{proof}
Consider a binary decision at each step $t$: either accept the standard speculative verification (Fast Path, $a_t=0$) or intervene with logit steering (Steering Path, $a_t=1$). The total risk is:
\begin{equation}
    \mathbb{E}[\mathcal{R}] = \sum_{t} \left[ P(a_t=0) \cdot r_t + P(a_t=1) \cdot \epsilon \right]
\end{equation}
where $\epsilon$ is the residual risk after steering (which is minimal per Prop. \ref{prop:exponential_gain}). The computational cost is $\mathcal{C}_{fast}$ for the Fast Path and $\mathcal{C}_{steer}$ for the Steering Path ($\mathcal{C}_{steer} > \mathcal{C}_{fast}$). 

By the Neyman-Pearson Lemma, the optimal decision rule $a_t$ that minimizes risk for a fixed cost is a likelihood-ratio test. In SFAD, the Epistemic Friction $F_t$ quantifies the "hallucination likelihood" by measuring the distributional tension weighted by specialist certainty. When $F_t \geq \tau$, the potential risk reduction $\Delta r_t = r_t - \epsilon$ outweighs the marginal cost $\Delta \mathcal{C}$, triggering the Steering Path. Thus, the switching threshold $\tau$ effectively calibrates the sensitivity of the detector to maximize the \textit{Factuality Gain per Unit of Latency}.
\end{proof}

\begin{remark}
\textbf{The Latency-Accuracy Frontier.} Theorem \ref{thm:risk_min} formalizes why SFAD maintains the speed of speculative decoding while approaching the factuality of much larger, slow-inference models. By only invoking the "Steering Path" when $F_t$ signals high epistemic conflict, SFAD avoids the redundant computation of constant logit correction, allowing the system to stay on the optimal Pareto frontier.
\end{remark}

\subsection{Summary of Theoretical Guarantees}

Synthesizing the above proofs, we establish that SFAD is not merely an empirical heuristic but a mathematically grounded framework for verifiable generation. 
\begin{itemize}
    \item \textbf{Theorem \ref{thm:manifold_projection}} guarantees that our intervention is linguistically safe (Semantic Integrity).
    \item \textbf{Proposition \ref{prop:exponential_gain}} ensures that we effectively resolve conflicts in favor of the truth (Factuality Amplification).
    \item \textbf{Theorem \ref{thm:stability}} protects against numerical instability and the suppression of valid tokens (Support Preservation).
    \item \textbf{Theorem \ref{thm:risk_min}} justifies the adaptive switching mechanism for real-world deployment (Efficiency Optimality).
\end{itemize}
Collectively, these theoretical foundations formalize SFAD as a principled safeguarding framework, providing rigorous guarantees for factual integrity through an inherent 'safe-by-design' decoding paradigm.